\documentclass[letterpaper, 10 pt, conference]{ieeeconf}

\usepackage[utf8]{inputenc}
\usepackage[T1]{fontenc}

\IEEEoverridecommandlockouts
\usepackage{graphics}
\usepackage{epsfig}
\usepackage{mathptmx}
\usepackage{times}
\usepackage{amsmath}
\usepackage{amssymb}
\usepackage{mathrsfs}
\usepackage{algorithm}
\usepackage{algorithmic}
\usepackage{graphicx}
\usepackage{booktabs}
\usepackage{array}
\usepackage{caption}
\usepackage{multirow}
\usepackage{cite}
\usepackage{bm}
\usepackage{dblfloatfix}
\newtheorem{assumption}{Assumption}
\newtheorem{proposition}{Proposition}
\newtheorem{remark}{Remark}

\title{\LARGE \bf
Learning to Adapt: Reptile-D-Learning for Robust and Efficient Control Under Parametric Uncertainty
}

\author{Haipeng Cao, Zhaolong Shen and Quan Quan*
\thanks{This work was supported in part by the National Natural Science Foundation of China under Grant 62573021.}
\thanks{The authors are with the School of Automation Science and Electrical Engineering, Beihang University, Beijing 100191, P.R. China.
        {\tt\small \{20375039, shenzhaolong, qq\_buaa\}@buaa.edu.cn}}
\thanks{*Corresponding author.}
}

\begin{document}

\maketitle
\thispagestyle{empty}
\pagestyle{empty}

\begin{abstract}

Learning-based Lyapunov Control (LLC) provides formal stability guarantees for nonlinear systems, but its validity relies heavily on accurate system models. Parameter variations and uncertainties may invalidate stability constraints, leading to costly retraining. Although D-learning can estimate Lyapunov derivatives without relying on explicit dynamics models, it remains limited by single-task dynamics and degrades under large parameter shifts. We propose Reptile-D-learning, a framework that leverages the Reptile meta-learning algorithm to capture shared dynamical structures across systems with different parameters, thereby learning a generalizable Lyapunov network initialization and a high-performance controller. Experiments on multiple nonlinear control systems demonstrate that Reptile-D-learning significantly improves both generalization and rapid adaptation to unseen parameter configurations.

\end{abstract}


\section{INTRODUCTION}

Learning-based Lyapunov Control (LLC) methods \cite{1, 3, 5, 9, berkenkamp2017} have emerged as a principled framework by jointly learning Lyapunov functions and control policies. These methods provide formal stability guarantees for neural network controllers, thereby addressing the inherent opacity and lack of safety certificates in deep reinforcement learning, but such guarantees can be brittle in dynamic environments. The standard Lyapunov condition $\dot{V}(x) < 0$ is inherently tied to a specific nominal dynamics model $f(x, u)$. In practical robotic applications, physical parameters such as mass, friction coefficients, and payloads are often time-varying or uncertain, which can distort the underlying energy landscape, invalidate learned stability certificates, and necessitate costly retraining \cite{6, 8}.

Although D-learning \cite{10} relaxes model dependence by using an auxiliary network $D(x,u)$ for model-free Lyapunov derivative estimation, it often overfits to single-task dynamics. Significant parameter shifts cause dynamics mismatch, leading to biased derivative estimates and degraded control \cite{11,12}, as illustrated in Fig.~\ref{fig:hero}. This is especially critical during online operation, where data collection is limited, expensive, and time-sensitive.

These observations motivate a cross-parameter learning principle: instead of learning for a fixed instance, we learn an initialization capturing shared structures across task distributions for rapid adaptation. This mechanism must preserve stability objectives while remaining computationally tractable for the cascaded Lyapunov, Dfunction, and policy networks.

To address these challenges, we propose \textbf{Reptile-D-learning}, a unified bilevel framework for robust cross-parameter stabilization. We represent the Lyapunov network, D-network, and control policy as a joint meta-parameter $\Theta=\{\theta_V,\psi_D,\phi\}$, and optimize the expected post-adaptation stability loss over a task distribution. Reptile \cite{17} is adopted as a first-order approximate solver for this bilevel objective. Compared with second-order methods such as MAML \cite{16}, this design avoids expensive dense block-Hessian computations and is therefore better suited to D-learning's three-network coupled constraints.

\begin{figure}[t]
    \centering
    \includegraphics[width=\columnwidth]{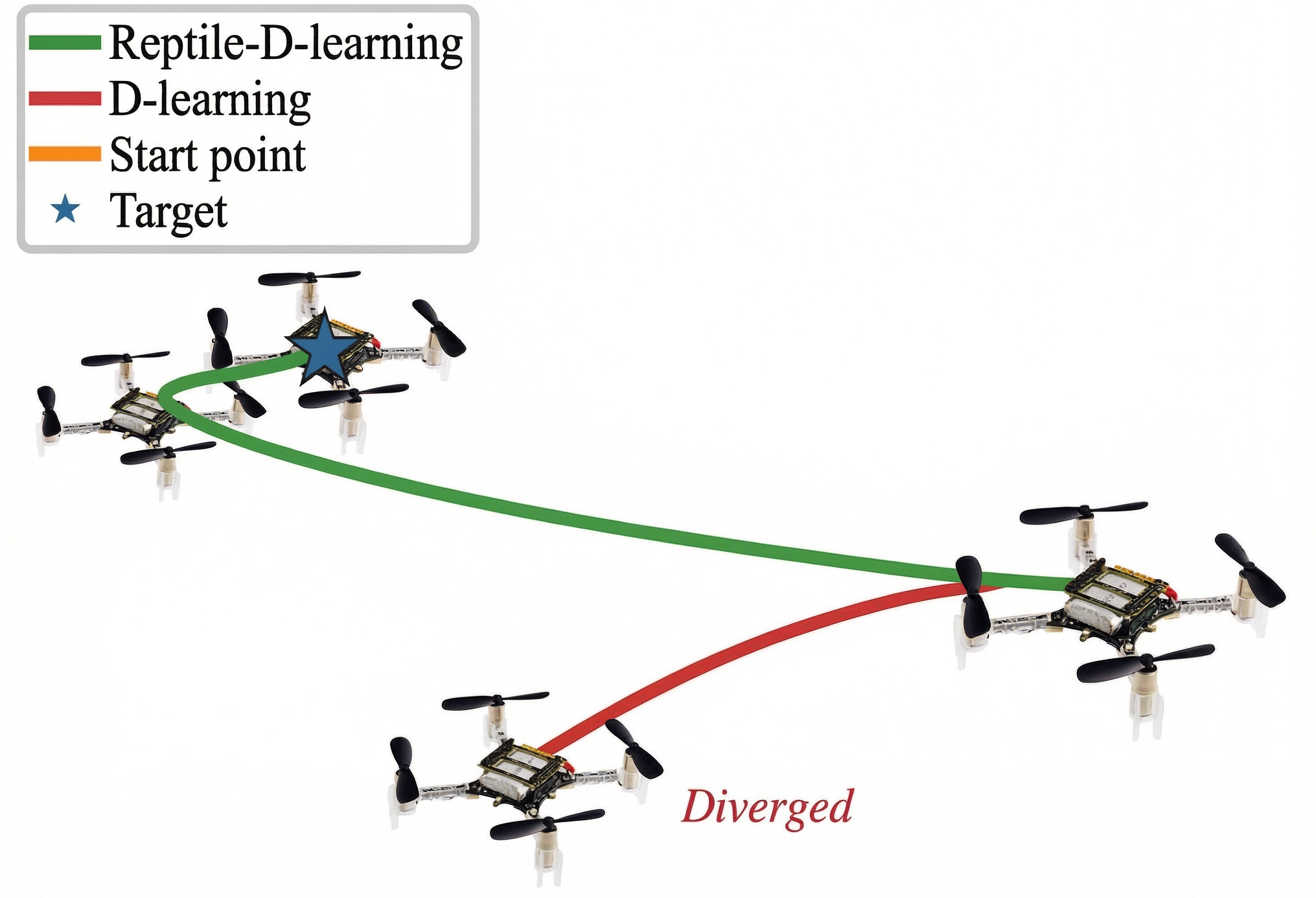}
    \caption{3D trajectory comparison under 1.5$\times$ mass perturbation. The D-learning controller diverges, while Reptile-D-learning converges to the target.}
    \label{fig:hero}
  \end{figure}

The main contributions are summarized as follows:
\begin{itemize}
    \item We formulate cross-parameter LLC adaptation as a \textbf{unified bilevel optimization} problem over the joint meta-parameter $\Theta$. We cast D-learning as the task-specific inner solver and Reptile as the first-order meta-optimizer.
    \item We demonstrate that Reptile is a \textbf{computationally efficient first-order solver} specifically suited for the cascaded constraints of D-learning. Furthermore, we provide a gradient-level analysis showing that the meta-update preserves cross-task gradient consistency under a dynamics decomposition.
    \item We evaluate \textbf{Reptile-D-learning} on three nonlinear benchmarks, showing improved generalization, rapid limited-data adaptation, and stable control performance.
\end{itemize}

\begin{figure*}[t]
    \centering
    \includegraphics[width=\textwidth]{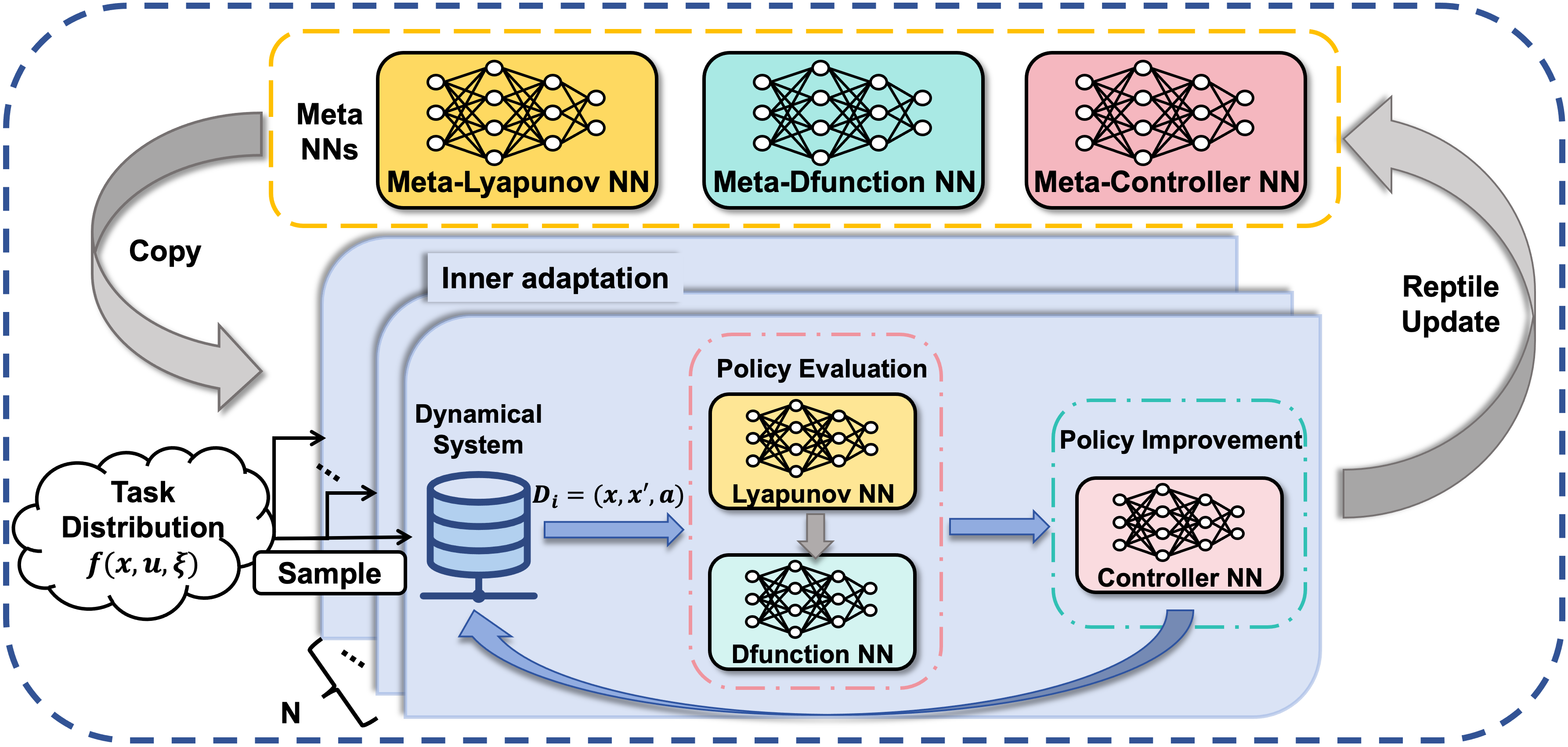}
    \caption{Overview of the Reptile-D-learning framework. Three meta-networks (Meta-Lyapunov NN, Meta-Dfunction NN, Meta-Controller NN) are copied to initialize $N$ parallel sub-tasks sampled from the task distribution $p(\Xi)$. Each sub-task runs inner-loop adaptation via D-learning: the Policy Evaluation phase trains the Lyapunov and Dfunction networks, and the Policy Improvement phase updates the controller. The adapted parameters are aggregated via the Reptile meta-update to refine the meta-networks.}
    \label{fig:overview}
  \end{figure*}

\section{RELATED WORK}

\noindent\textbf{Learning-based Lyapunov Control.}
LLC methods jointly learn control policies and Lyapunov functions to provide verifiable stability guarantees for nonlinear systems \cite{1, 2, 3, 4, 5, berkenkamp2017}.
Existing approaches typically embed stability conditions into policy optimization or constrained learning pipelines
to ensure closed-loop stability alongside control performance \cite{7, 9}. However, these methods generally rely on a single nominal dynamics model
or a fixed parameter distribution. When parameters such as mass, friction, or payload drift, the original Lyapunov certificate can become invalid,
often necessitating costly retraining on the new task \cite{6, 8}.

\noindent\textbf{D-learning.}
To reduce the dependency on explicit dynamics models,
D-learning introduces a Dfunction network to estimate the Lyapunov derivative,
thereby decoupling stability constraints from dynamics modeling \cite{10}.
This approach offers advantages for black-box systems. However,
existing D-learning methods are primarily designed for single-task or narrow-distribution training.
When facing large distribution shifts across parameters, the derivative estimation network can exhibit systematic mismatch,
degrading both the transferability of stability constraints and control performance \cite{11, 12}.

\noindent\textbf{Meta-Learning for Adaptive Control.}
Meta-reinforcement learning has been applied to initialize policies, value functions, or system identification modules for control under parametric uncertainty \cite{14, nagabandi2019, hospedales2021}, significantly reducing fine-tuning steps in new environments. More broadly, transfer-learning and reward-driven RL paradigms, from survey-level perspectives to classical value-based methods, emphasize adaptation efficiency \cite{13, 21}. However, most of these methods optimize reward without explicit Lyapunov or stability constraints. Recent works \cite{15} begin to bridge this gap by integrating MAML with Neural Lyapunov Functions for adaptive stability certification, yet they require explicit dynamics models, learn only the Lyapunov function without jointly optimizing the controller, and rely on second-order derivatives that incur substantial computational overhead when the inner-loop involves multiple interacting networks. Unlike meta-NLF \cite{15}, which requires explicit dynamics and learns only $V$, and unlike reward-focused meta-RL \cite{14, nagabandi2019}, which provides no stability certificates, the proposed Reptile-D-learning jointly meta-learns $V$, $D$, and $\pi$ in a model-free, stability-guaranteed framework.

\section{PRELIMINARIES}

\subsection{Lyapunov Stability and Parametric Uncertainty}

Consider a closed-loop dynamical system with parametric uncertainty:
\begin{equation}
\dot{x} = f(x, u; \xi), \quad u = \pi_\phi(x), \quad x(0) = x_0,
\label{eq:system}
\end{equation}
where $x \in \mathcal{X} \subseteq \mathbb{R}^n$ is the state, $u \in \mathcal{U} \subseteq \mathbb{R}^m$ is the control input, $f$ is the Lipschitz-continuous dynamics, $\pi_\phi$ is a parameterized control policy, and $\xi \in \Xi \subseteq \mathbb{R}^p$ represents unknown or time-varying physical parameters (e.g., mass, friction, payload).

A continuously differentiable function $V(x)$ defined on an open neighborhood $\mathcal{D}$ of the origin is a \textit{Lyapunov function} if it satisfies:
\begin{equation}
V(0) = 0, \quad V(x) > 0, \quad \dot{V}(x) < 0, \quad \forall x \in \mathcal{D} \setminus \{0\},
\label{eq:lyapunov}
\end{equation}
where $\dot{V}(x) = \nabla_x V(x)^\top f(x, \pi_\phi(x); \xi)$. Satisfaction of \eqref{eq:lyapunov} certifies asymptotic stability of the origin. Crucially, this condition is tied to a \textit{specific} $\xi$; when $\xi$ changes, the certificate may be invalidated.

\subsection{D-learning and Dynamics Encoding}

D-learning \cite{10} introduces a Dfunction network $D_{\psi_D}(x, u)$ to approximate the Lyapunov derivative without explicit knowledge of $f$:
\begin{equation}
D_{\psi_D}(x, u) \approx \nabla_x V_{\theta_V}(x)^\top f(x, u; \xi).
\label{eq:D_def}
\end{equation}
By fitting $D_{\psi_D}$ to the discrete-time difference of the learned Lyapunov function $V_{\theta_V}$, the D-network implicitly encodes the system dynamics. The stability condition $\dot{V}(x) < 0$ is then enforced through $D_{\psi_D}(x, \pi_\phi(x)) < 0$, decoupling stability verification from model knowledge.

D-learning operates in a \textit{policy evaluation--policy improvement} loop: the evaluation phase trains $V_{\theta_V}$ and $D_{\psi_D}$ on sampled trajectories, and the improvement phase uses $D_{\psi_D}$ to guide the controller $\pi_\phi$ toward faster convergence. We denote the joint parameter of all three networks as
\begin{equation}
\Theta = \{\theta_V,\; \psi_D,\; \phi\},
\label{eq:joint_param}
\end{equation}
and the task-level loss as $L_{\mathcal{T}_\xi}(\Theta)$ (detailed in Section~IV-A).

This dynamics encoding property has a direct consequence for cross-parameter generalization. Decomposing the dynamics into a shared component and a task-specific perturbation,
\begin{equation}
f(x, u; \xi_i) = \bar{f}(x, u) + \delta_i(x, u; \xi_i),
\label{eq:dynamics_decompose}
\end{equation}
where $\bar{f}(x,u) = \mathbb{E}_{\xi}[f(x,u;\xi)]$ is the mean dynamics across the parameter distribution and $\delta_i$ is the residual for task $\xi_i$, the Dfunction decomposes accordingly:
\begin{equation}
D_{\psi_D}(x, u) \approx \underbrace{\nabla_x V_{\theta_V}(x)^\top \bar{f}(x, u)}_{\bar{D}:\;\text{shared component}}
+ \underbrace{\nabla_x V_{\theta_V}(x)^\top \delta_i(x, u)}_{\Delta D_i:\;\text{task-specific}}.
\label{eq:D_decompose}
\end{equation}
This decomposition reveals that adapting D-learning to a new parameter configuration amounts to calibrating the task-specific component $\Delta D_i$ from a shared baseline $\bar{D}$. If a meta-initialization captures $\bar{D}$ well, the remaining adaptation focuses on residual correction rather than relearning the full dynamics.

\subsection{Cross-Parameter Adaptation Problem}

We treat each parameter configuration $\xi_i \sim p(\Xi)$ as defining an independent control task $\mathcal{T}_{\xi_i}$. Each task involves the same network architecture $\Theta$ but is governed by distinct dynamics $f(\cdot; \xi_i)$. The goal is to learn a meta-initialization $\Theta^*$ such that a small number of task-level updates yields a stabilizing controller and reliable Lyapunov certificates on newly sampled tasks. As a concrete example, in the UAV system with mass and inertia shifted from nominal values (e.g., $1.5\times$), adaptation from $\Theta^*$ should reach high success rates and fast convergence in substantially fewer steps than retraining standard D-learning on the shifted system.

This is formalized as the bilevel optimization:
\begin{equation}
\min_{\Theta}\; \mathcal{J}(\Theta)
= \mathbb{E}_{\xi \sim p(\Xi)}\!\left[L_{\mathcal{T}_\xi}\!\left(\Theta_\xi^K\right)\right],
\label{eq:bilevel}
\end{equation}
Here, $L_{\mathcal{T}_\xi}$ denotes the task-level training objective for task $\mathcal{T}_\xi$. Its explicit decomposition and optimization procedure are presented in Section~IV-A.
where $\Theta_\xi^K$ denotes the parameters obtained after $K$ steps of D-learning adaptation on task $\mathcal{T}_\xi$ starting from $\Theta$:
\begin{equation}
\Theta_\xi^0 = \Theta, \quad
\Theta_\xi^{k+1} = \Theta_\xi^k - \alpha\, \nabla_\Theta L_{\mathcal{T}_\xi}(\Theta_\xi^k).
\label{eq:inner_update}
\end{equation}

Computing the exact gradient of \eqref{eq:bilevel} via MAML \cite{16} requires differentiating through the $K$-step adaptation, involving the Hessian of $L_{\mathcal{T}_\xi}$ with respect to $\Theta$. As we show in Section~IV-B, the cascading dependencies in D-learning produce a dense block Hessian, making this computation intractable. Instead, we adopt Reptile \cite{17} as a first-order approximate solver:
\begin{equation}
\Theta \leftarrow \Theta + \epsilon \cdot \frac{1}{N}\sum_{i=1}^{N}(\Theta_{\xi_i}^K - \Theta),
\label{eq:reptile}
\end{equation}
which requires only the parameter difference before and after adaptation. In Section~IV, we show that this first-order update not only approximates the same meta-objective as MAML, but also preserves cross-task gradient consistency under the dynamics decomposition~\eqref{eq:dynamics_decompose}.


\section{PROPOSED METHOD}

An overview of the proposed Reptile-D-learning framework is shown in Fig.~\ref{fig:overview}.

\subsection{Task-Level D-learning Losses}

For each task $\mathcal{T}_\xi$, the adaptation in \eqref{eq:inner_update} minimizes the joint loss $L_{\mathcal{T}_\xi}(\Theta) = L_V + L_D + L_\pi$. Given a dataset $\{(x_i, x_i', a_i)\}_{i=1}^M$ collected under $\pi_\phi$ with $x_i' = x_i + \Delta t \cdot f(x_i, a_i; \xi)$, the three components are:

\noindent\textbf{Lyapunov loss.} Enforces positive-definiteness and derivative negativity:
\begin{equation}
L_V(\theta_V)
= \frac{1}{M}\sum_{i=1}^{M}\!\left[\max(-V_{\theta_V}(x_i), 0)
+ \max(\dot{V}_{\theta_V}(x_i), 0)\right]
+ V_{\theta_V}^2(0),
\label{eq:loss_V}
\end{equation}
where $\dot{V}_{\theta_V}(x_i) = (V_{\theta_V}(x_i') - V_{\theta_V}(x_i))/\Delta t$.

\noindent\textbf{Dfunction loss.} Fits $D_{\psi_D}$ to the discrete Lyapunov derivative:
\begin{equation}
L_D(\psi_D, \theta_V)
= \frac{1}{M}\sum_{i=1}^{M}\!\left(\frac{D_{\psi_D}(x_i, a_i) - \dot{V}_{\theta_V}(x_i)}{\|x_i\|^2 + \lambda}\right)^{\!2}
\!+ D_{\psi_D}^2(0, 0),
\label{eq:loss_D}
\end{equation}
where $\lambda > 0$ prevents division by zero and $D_{\psi_D}^2(0,0)$ enforces the equilibrium boundary condition $\dot{V}(0) = 0$.

\noindent\textbf{Controller loss.} Drives $\pi_\phi$ to minimize the Dfunction, i.e., to ensure $\dot{V} < 0$:
\begin{equation}
L_\pi(\phi, \psi_D)
= \frac{1}{M}\sum_{i=1}^{M}\!\left[\frac{D_{\psi_D}(x_i, \pi_\phi(x_i))}{\|x_i\|^2 + \lambda'}
+ \max\!\big(D_{\psi_D}(x_i, \pi_\phi(x_i)),\, 0\big)\right].
\label{eq:loss_pi}
\end{equation}

\noindent\textbf{Cascading dependency.} These losses are \textit{coupled} $L_D$ depends on $V_{\theta_V}$ through its supervision signal $\dot{V}_{\theta_V}$, and $L_\pi$ depends on $D_{\psi_D}$ through its evaluation metric. This cascading structure directly implies the choice of meta-optimizer, as analyzed next.

\subsection{Reptile as First-Order Solver}

To optimize the bilevel objective \eqref{eq:bilevel}, MAML computes the meta-gradient by differentiating through the $K$-step adaptation. For $K=1$:
\begin{equation}
\nabla_\Theta \mathcal{J}_{\text{MAML}}
= \mathbb{E}_\xi\!\left[\left(I - \alpha\, H_\xi\right) \nabla_{\Theta_\xi^1} L_{\mathcal{T}_\xi}(\Theta_\xi^1)\right],
\label{eq:maml_grad}
\end{equation}
where $H_\xi = \nabla_\Theta^2 L_{\mathcal{T}_\xi}(\Theta)$ is the Hessian. Due to the cascading dependency, $H_\xi$ has a dense block structure with non-zero off-diagonal blocks:
\begin{equation}
H_\xi =
\begin{pmatrix}
\frac{\partial^2 L}{\partial \theta_V^2} & \frac{\partial^2 L_D}{\partial \theta_V\, \partial \psi_D} & \bm{0} \\
\frac{\partial^2 L_D}{\partial \psi_D\, \partial \theta_V} & \frac{\partial^2 L}{\partial \psi_D^2} & \frac{\partial^2 L_\pi}{\partial \psi_D\, \partial \phi} \\
\bm{0} & \frac{\partial^2 L_\pi}{\partial \phi\, \partial \psi_D} & \frac{\partial^2 L}{\partial \phi^2}
\end{pmatrix}\!.
\label{eq:hessian}
\end{equation}
The cross-blocks arise because $L_D$ couples $(\theta_V, \psi_D)$ and $L_\pi$ couples $(\psi_D, \phi)$. MAML incurs substantially higher cost due to repeated Hessian-vector products.

Reptile avoids explicit Hessian computation entirely. A Taylor expansion of the $K=2$ Reptile update yields \cite{17}:
\begin{equation}
\mathbb{E}_\xi[\Theta_\xi^2 - \Theta]
= \underbrace{-2\alpha\, \bar{g}}_{\text{average gradient}}
+ \underbrace{\alpha^2 \mathbb{E}_\xi[H_\xi\, g_\xi^0]}_{\text{implicit second-order term}}
+ O(\alpha^3),
\label{eq:reptile_taylor}
\end{equation}
where $\bar{g} = \mathbb{E}_\xi[\nabla_\Theta L_{\mathcal{T}_\xi}]$ and $g_\xi^0 = \nabla_\Theta L_{\mathcal{T}_\xi}(\Theta)$. The second-order term $\alpha^2 H_\xi g_\xi^0$ matches the MAML gradient \eqref{eq:maml_grad} structure but is produced \textit{implicitly} via forward adaptation.

\begin{remark}[Why not other meta-learning methods?]
\label{rmk:alternatives}
(i)~\textbf{MAML} requires the dense Hessian \eqref{eq:hessian}, incurring substantially higher computational cost.
(ii)~\textbf{FO-MAML} implementation complexity is higher due to alternating inner updates rather than a single inner objective descent.
(iii)~\textbf{iMAML} \cite{rajeswaran2019} avoids Hessians via implicit differentiation but requires inner convergence to a fixed point, conflicting with our alternating structure.
(iv)~\textbf{Latent methods} learn task embeddings but lack stability certificates; conditioning $V_{\theta_V}(x; z)$ on $z$ offers no guarantee that \eqref{eq:lyapunov} holds universally.
\end{remark}

\subsection{Gradient Consistency under Dynamics Decomposition}

We now show Reptile's meta-update preserves cross-task gradient consistency, providing formal justification beyond computational convenience.

\begin{assumption}[Task perturbation structure]
\label{assump:perturbation}
The task-specific perturbations $\delta_i$ in \eqref{eq:dynamics_decompose} satisfy:
(i)~Zero mean: $\mathbb{E}_{\xi}[\delta_i(x,u;\xi)] = 0$;
(ii)~Weak correlation: $\mathbb{E}_{\xi_i \neq \xi_j}[\delta_i^\top \delta_j] \approx 0$;
(iii)~\emph{Gradient sensitivity:} $\nabla_\Theta L_{\delta_i} \approx J_\Theta\delta_i$ on $\mathcal{D}$, where $J_\Theta$ is a bounded, approximately task-invariant sensitivity Jacobian.
\end{assumption}

Under decompositions \eqref{eq:dynamics_decompose} and \eqref{eq:D_decompose}, the task gradient decomposes as $\nabla_\Theta L_{\mathcal{T}_{\xi_i}} = \nabla_\Theta L_{\bar{f}}(\Theta) + \nabla_\Theta L_{\delta_i}(\Theta)$, where $L_{\bar{f}}$ corresponds to shared dynamics and $L_{\delta_i}$ to the residual. Expanding the cross-task gradient inner product:
\begin{align}
&\left\langle \nabla_\Theta L_{\mathcal{T}_{\xi_i}},\;
\nabla_\Theta L_{\mathcal{T}_{\xi_j}} \right\rangle
= \left\|\nabla_\Theta L_{\bar{f}}\right\|^2 \notag\\
&\quad+ \left\langle \nabla_\Theta L_{\bar{f}},\, \nabla_\Theta L_{\delta_i} + \nabla_\Theta L_{\delta_j}\right\rangle
+ \left\langle \nabla_\Theta L_{\delta_i},\, \nabla_\Theta L_{\delta_j} \right\rangle.
\label{eq:inner_product}
\end{align}
Condition~(iii) bridges dynamics residuals to parameter-space gradients via $\nabla_\Theta L_{\delta_i}\approx J_\Theta\delta_i$, so conditions~(i)--(ii) imply the cross-terms in \eqref{eq:inner_product} vanish in expectation:

\begin{proposition}[Cross-task gradient consistency]
\label{prop:consistency}
Under Assumption~\ref{assump:perturbation},
\begin{equation}
\mathbb{E}_{\xi_i \neq \xi_j}\!\left[\left\langle \nabla_\Theta L_{\mathcal{T}_{\xi_i}},\; \nabla_\Theta L_{\mathcal{T}_{\xi_j}} \right\rangle\right]
\approx \left\|\nabla_\Theta L_{\bar{f}}\right\|^2.
\label{eq:core_result}
\end{equation}
\end{proposition}
\begin{proof}
By Assumption~\ref{assump:perturbation}(iii), $\nabla_\Theta L_{\delta_i}\approx J_\Theta\delta_i$.
Condition~(i) gives $\mathbb{E}[\nabla_\Theta L_{\delta_i}]\approx J_\Theta\mathbb{E}[\delta_i]=0$, and
condition~(ii) gives $\mathbb{E}[\langle\nabla_\Theta L_{\delta_i}, \nabla_\Theta L_{\delta_j}\rangle]\approx\mathbb{E}[\delta_i^\top J_\Theta^\top J_\Theta\delta_j]\approx 0$.
Substituting into~\eqref{eq:inner_product} yields~\eqref{eq:core_result}.
\end{proof}

Since Reptile's meta-update implicitly increases gradient agreement \cite{17}, Proposition~\ref{prop:consistency} suggests the update is primarily influenced by shared dynamics $\bar{f}$, consistent with the Dfunction decomposition \eqref{eq:D_decompose}.

\begin{figure*}[t]
    \centering
    \includegraphics[width=\textwidth]{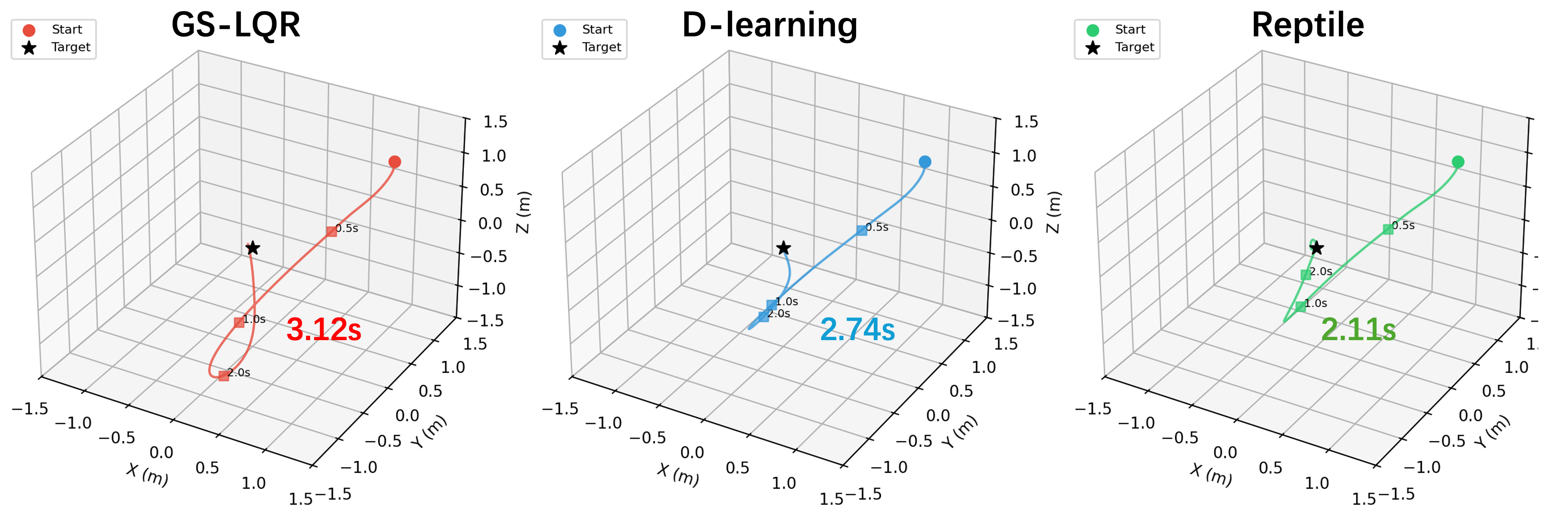}
    \caption{UAV stabilization trajectories under the benchmark system. Compared with the baseline and standard D-learning, Reptile-D-learning achieves faster and more reliable convergence to the target.}
    \label{fig:uav_compare_new}
\end{figure*}

\begin{figure*}[t]
    \centering
    \begin{minipage}{0.49\textwidth}
        \centering
        \includegraphics[width=\linewidth]{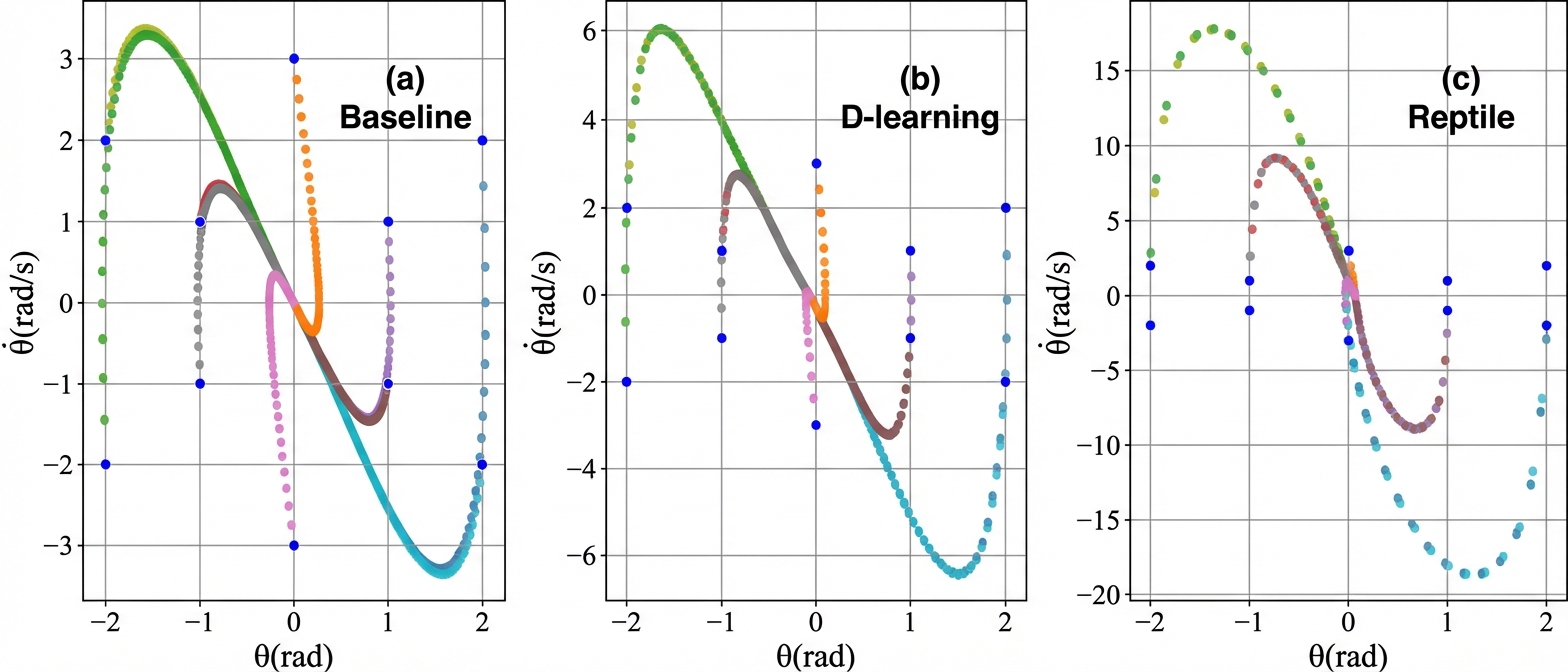}
        \captionof{figure}{Phase-portrait comparison for the inverted pendulum under the benchmark system. Reptile-D-learning demonstrates superior convergence performance compared to the baseline and standard D-learning.}
        \label{fig:inv_compare_new}
    \end{minipage}
    \hfill
    \begin{minipage}{0.49\textwidth}
        \centering
        \includegraphics[width=\linewidth]{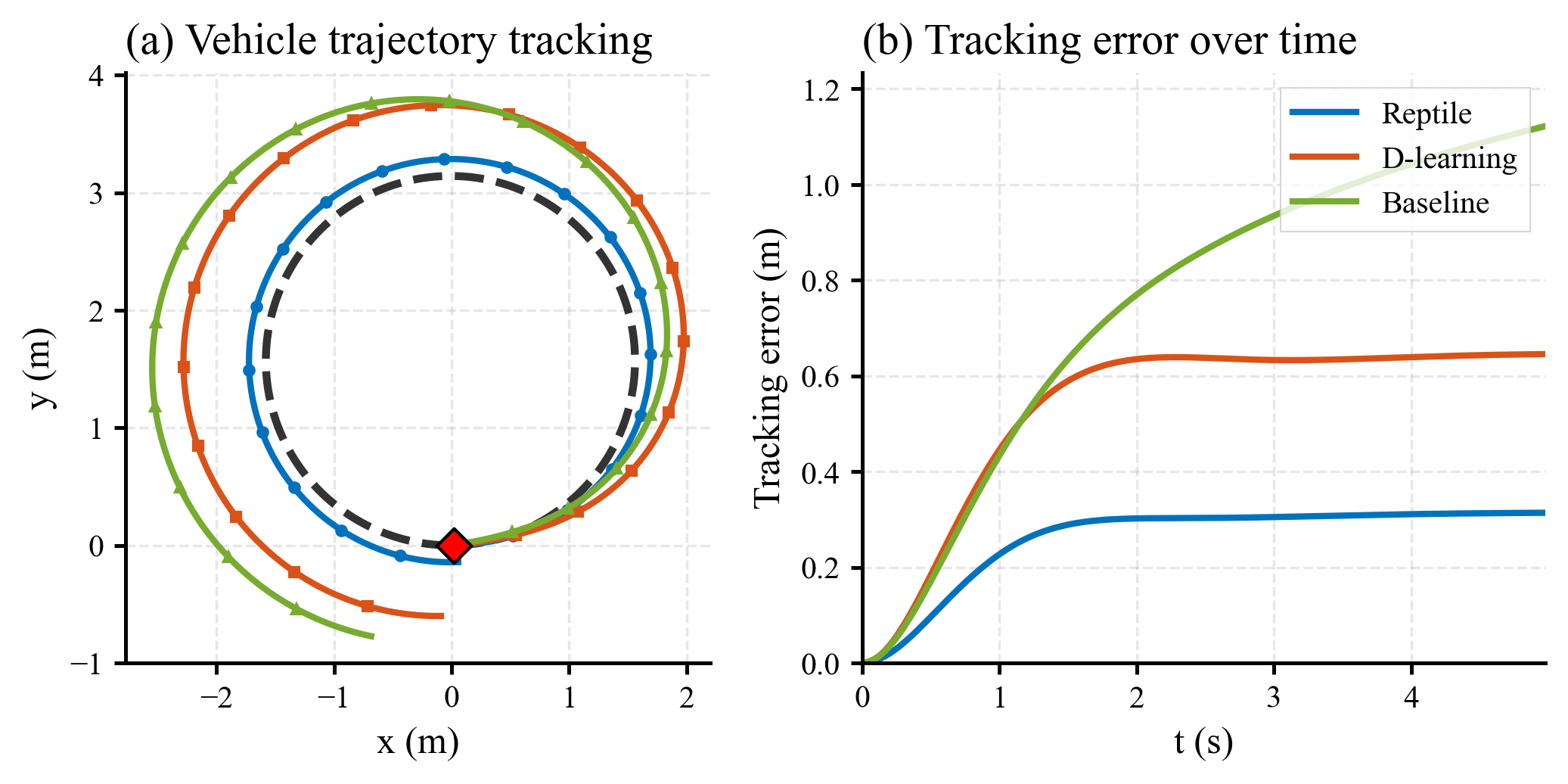}
        \captionof{figure}{Single-track car performance comparison. (a) Vehicle trajectory tracking; (b) tracking error over time. Reptile-D-learning achieves the smallest tracking error and closest trajectory to the target path.}
        \label{fig:stcar_compare_new}
    \end{minipage}
\end{figure*}

\subsection{Algorithm and Stability Analysis}

The complete Reptile-D-learning procedure is summarized in Algorithm~\ref{alg:reptile_d}.

\begin{algorithm}[t]
\caption{Reptile-D-learning}
\label{alg:reptile_d}
\begin{algorithmic}[1]
\REQUIRE Task distribution $p(\Xi)$, batch size $N$, steps $K$, rates $\epsilon, \alpha$
\ENSURE Meta-parameters $\Theta^* = \{\theta_V^*, \psi_D^*, \phi^*\}$
\STATE Initialize $\Theta = \{\theta_V, \psi_D, \phi\}$
\FOR{meta-iteration $t = 1, 2, \ldots, T$}
    \FOR{$j = 1, \ldots, N$}
        \STATE Sample $\xi_j \sim p(\Xi)$; set $\Theta_j \leftarrow \Theta$
        \FOR{$k = 1, \ldots, K$}
            \STATE Collect trajectories under $\pi_{\phi_j}$ in $f(\cdot;\xi_j)$
            \STATE \textit{Policy evaluation:} 
            \STATE \quad Update $\theta_{V,j}$ via $\nabla_{\theta_V} L_V$ \hfill\eqref{eq:loss_V}
            \STATE \quad Update $\psi_{D,j}$ via $\nabla_{\psi_D} L_D$ \hfill\eqref{eq:loss_D}
            \STATE \textit{Policy improvement:}
            \STATE \quad Update $\phi_j$ via $\nabla_\phi L_\pi$ \hfill\eqref{eq:loss_pi}
        \ENDFOR
        \STATE Store adapted parameters $\Theta_j^K$
    \ENDFOR
    \STATE \textbf{Reptile meta-update:}
    \STATE \quad $\Theta \leftarrow \Theta + \epsilon \cdot \frac{1}{N}\sum_{j=1}^{N}(\Theta_j^K - \Theta)$
\ENDFOR
\end{algorithmic}
\end{algorithm}

\noindent\textbf{Stability implications of meta-initialization.}
Suppose the controller satisfies $D_{\psi_D^*}(x, \pi_{\phi^*}(x)) < -\eta_0 \|x\|^2$ for $x \in \mathcal{D}\setminus\{0\}$ under $\bar{f}$. For a new task $\xi_\text{new}$, the actual derivative is:
\begin{align}
\dot{V}(x) &= D_{\psi_D^*}(x, \pi_{\phi^*}(x)) + \nabla_x V_{\theta_V^*}(x)^\top \delta_\text{new}(x) \notag\\
&< -\eta_0 \|x\|^2 + L_V \cdot \epsilon_\delta,
\label{eq:stability_bound}
\end{align}
where $L_V = \sup_{x \in \mathcal{D}} \|\nabla_x V_{\theta_V^*}(x)\|$ and $\epsilon_\delta = \sup_{x \in \mathcal{D}} \|\delta_\text{new}(x)\|$.

\begin{proposition}[Robustness of meta-initialization]
\label{prop:robust}
If there exists $c_\delta < \eta_0$ such that
\[
  \bigl|\nabla_x V_{\theta_V^*}(x)^\top \delta_\mathrm{new}(x)\bigr| \leq c_\delta\|x\|^2,
  \quad \forall\, x \in \mathcal{D}\setminus\{0\},
\]
then $\dot{V}(x)\leq-(\eta_0-c_\delta)\|x\|^2<0$, preserving Lyapunov negativity.
Otherwise, adaptation compensates the residual perturbation.
\end{proposition}
\begin{proof}
From \eqref{eq:stability_bound}, for all $x\in\mathcal{D}\setminus\{0\}$,
\begin{align*}
\dot{V}(x) &= D_{\psi_D^*} + \nabla_x V_{\theta_V^*}^\top\delta_\mathrm{new} \\
           &\leq -\eta_0\|x\|^2 + c_\delta\|x\|^2 = -(\eta_0-c_\delta)\|x\|^2 < 0,
\end{align*}
by standard Lyapunov perturbation arguments~\cite{khalil2002}.
\end{proof}

\noindent\textbf{Practical notes.}
(i)~Policy evaluation and improvement use separate data batches to avoid coupling bias. 
(ii)~Differential meta-learning rates $(\epsilon_V, \epsilon_D, \epsilon_\pi)$ are recommended. $V$ captures energy topology (less sensitive), while $D$ and $\pi$ encode dynamics details (more sensitive).

\section{EXPERIMENTS}

In this section, we assess Reptile-D-learning under cross-parameter uncertainty and compare it with baseline methods. We focus on three questions: controller performance and cross-parameter generalization, few-step adaptation, and the contribution of each component through ablation.

\begin{table}[t]
    \centering
    \caption{Benchmark Systems}
    \label{tab:exp_setup}
    \footnotesize
    \setlength{\tabcolsep}{2pt}
    \renewcommand{\arraystretch}{1.05}
    \begin{tabular}{p{0.20\columnwidth} p{0.28\columnwidth} p{0.43\columnwidth}}
    \toprule
    System & Varied parameters & Evaluation metrics \\
    \midrule
    Inverted Pendulum &
    \(\xi=(m,L,b)\) &
    Convergence steps \\
    \midrule
    Single-Track Car &
    \(\xi=(a,b,m,I_z,p_{dy1})\) &
    Mean tracking error \\
    \midrule
    UAV Stabilization &
    \(\xi=(m,I_{xx},I_{yy},I_{zz})\) &
    Convergence time; stabilization success rate \\
    \bottomrule
    \end{tabular}
    \end{table}

\subsection{Experimental Setup}

Experiments are conducted on three nonlinear systems with increasing complexity: an inverted pendulum, a single-track car error dynamics model from the CommonRoad benchmark \cite{19}, and a 3D UAV stabilization task using the Crazyflie nano-quadrotor platform \cite{giernacki2017, forster2015}. This design evaluates whether a single meta-initialization can transfer across tasks ranging from low-dimensional stabilization to high-dimensional tracking and control under parameter uncertainty.

We compare Reptile-D-learning against two baselines: an LQR controller (gain-scheduled LQR for the UAV task) and standard D-learning \cite{10}. All methods use the same network architecture to ensure a fair comparison.

Table~\ref{tab:exp_setup} summarizes the benchmark systems, uncertain parameters, and evaluation protocol. During training, uncertain parameters are sampled within prescribed ranges following a domain randomization strategy \cite{18, 20}. At test time, out-of-distribution (OOD) parameter configurations are introduced to evaluate generalization beyond the training distribution.

\begin{figure}[t]
    \centering
    \includegraphics[width=\columnwidth]{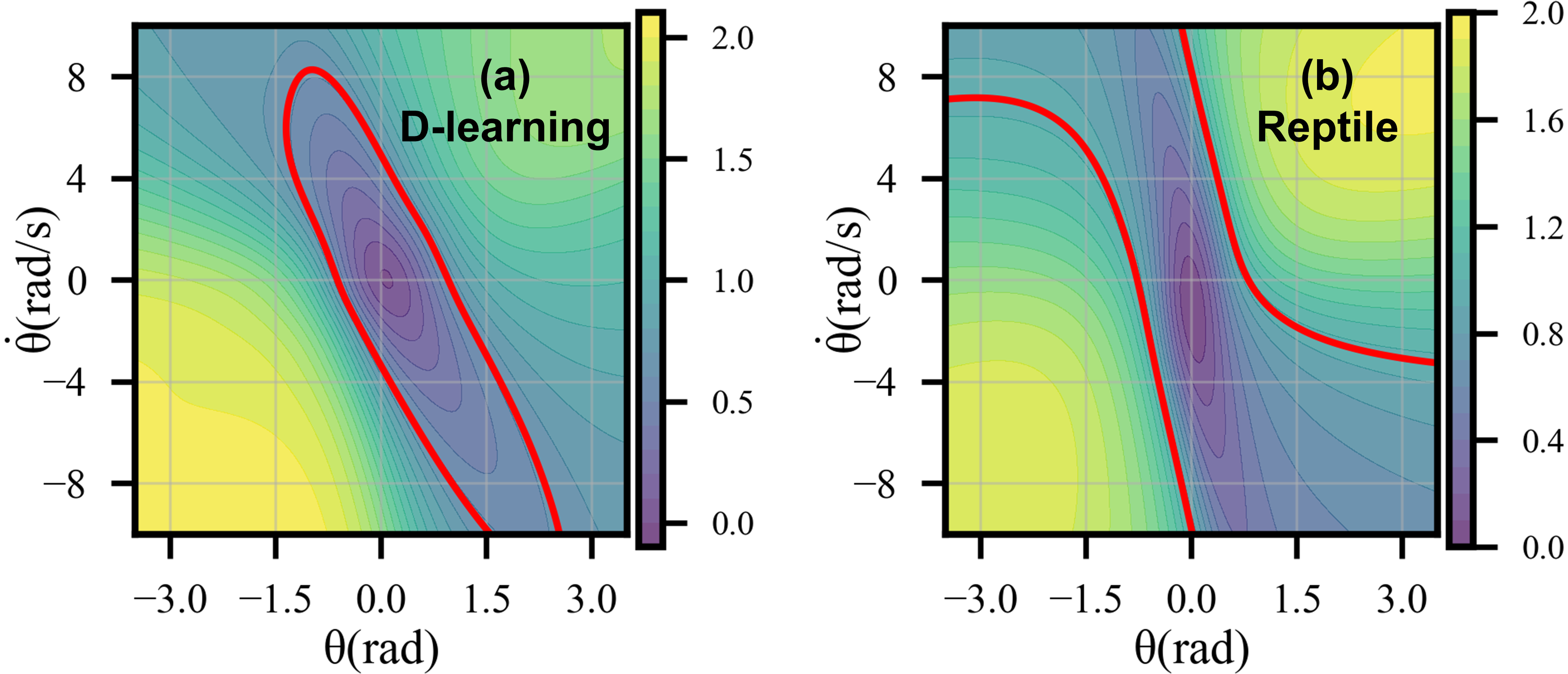}
    \caption{Comparison of estimated regions of attraction. Reptile method achieves a larger region with stronger gradients compared to standard D-learning.}
    \label{fig:roa_inv_new}
\end{figure}

\begin{figure}[t]
    \centering
    \begin{minipage}{\columnwidth}
        \centering
        \includegraphics[width=\columnwidth]{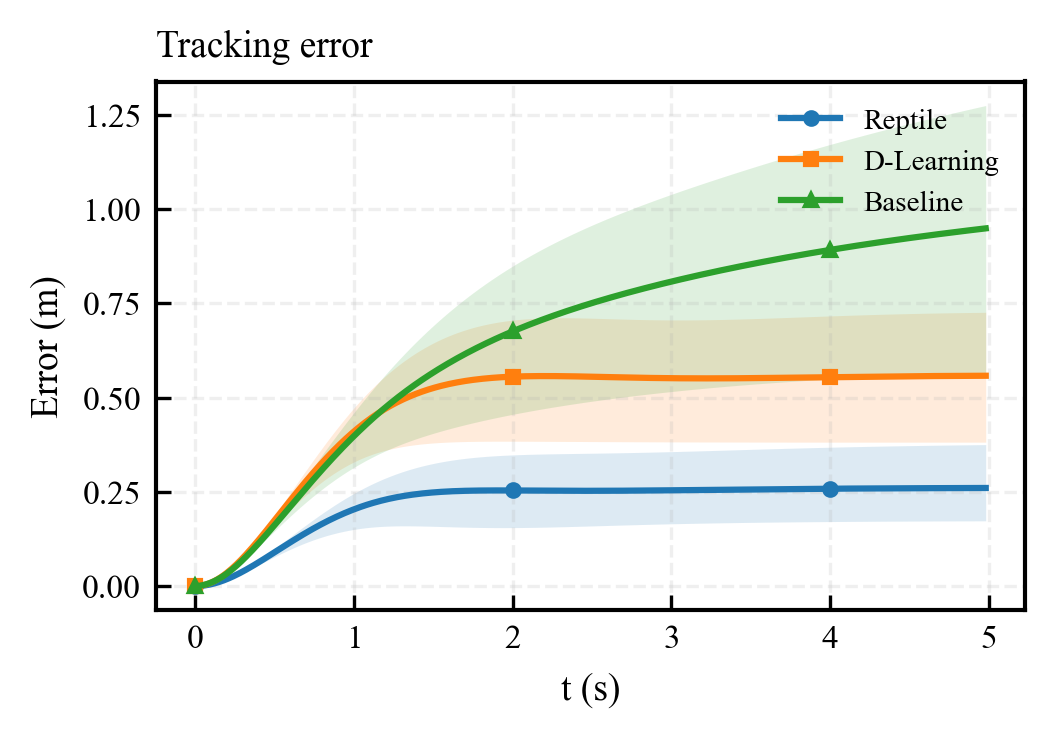}
        \captionof{figure}{Tracking error under parameter variation. Reptile-D-learning achieves the lowest error and variance compared with D-learning and baseline.}
        \label{fig:tracking_generalization}
    \end{minipage}
    \vspace{4pt}
    \begin{minipage}{\columnwidth}
        \centering
        \includegraphics[width=\columnwidth]{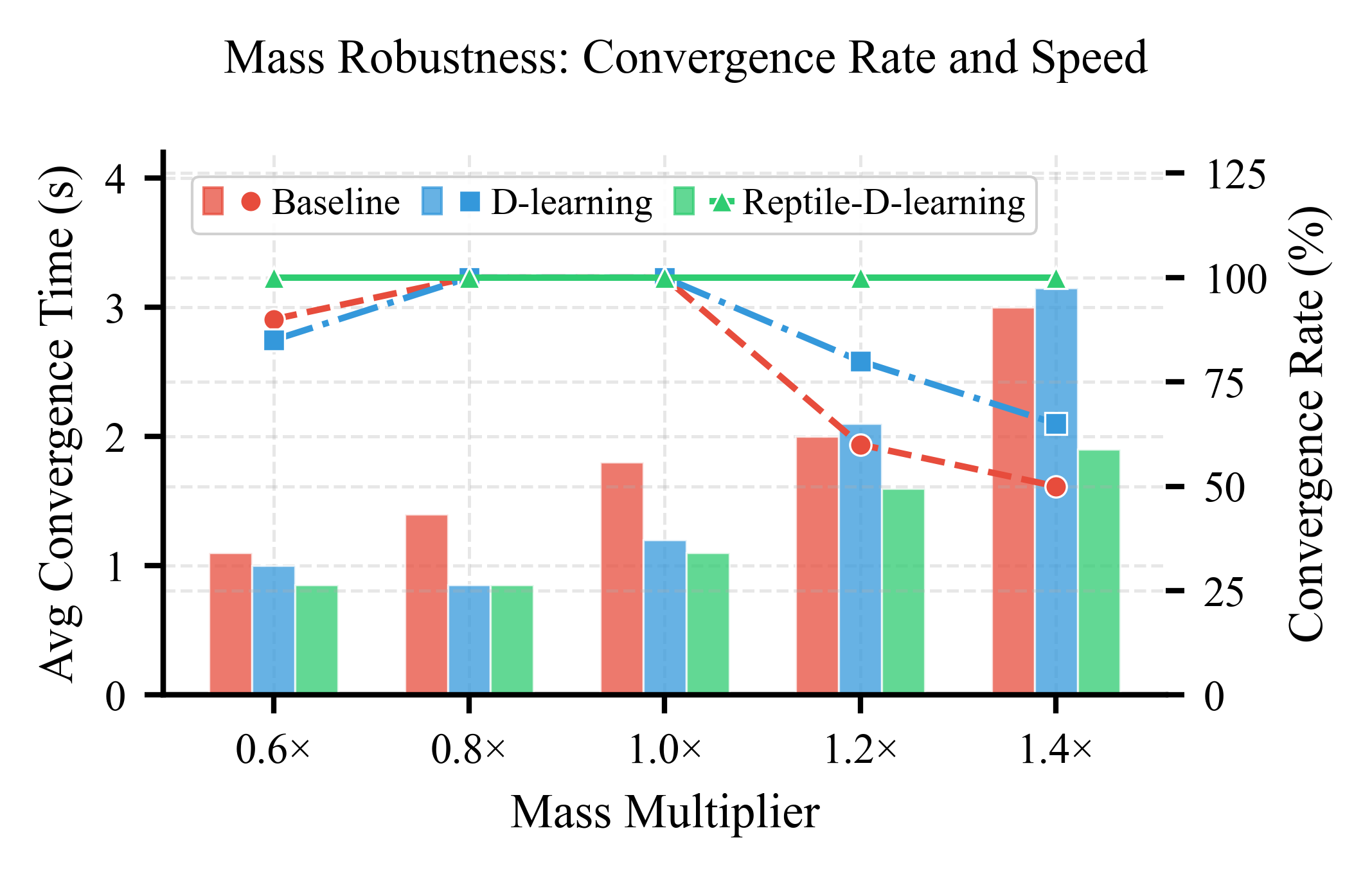}
        \captionof{figure}{Generalization under UAV mass shifts. Reptile-D-learning maintains higher convergence rate and lower convergence time as perturbation magnitude increases.}
        \label{fig:mass_robustness}
    \end{minipage}
\end{figure}

\captionsetup[table]{labelsep=newline}
\begin{table}[t]
\centering
\caption{Performance Under Different Parameter Configurations}
\label{tab:training_efficiency}
\scriptsize
\setlength{\tabcolsep}{2.5pt}
\renewcommand{\arraystretch}{1.05}
\resizebox{\columnwidth}{!}{%
\begin{tabular}{ccccccc}
\toprule
\multicolumn{3}{c}{\textbf{Parameters}} & \multirow{2}{*}{\textbf{Approach}} & \multicolumn{2}{c}{\textbf{Initial States}} & \multirow{2}{*}{\textbf{Convergence Steps}} \\
\cmidrule(r){1-3}\cmidrule(l){5-6}
\textbf{m} & \textbf{L} & \textbf{b} & & $\theta$ & $\dot{\theta}$ & \\
\midrule
\multirow{2}{*}{2.0} & \multirow{2}{*}{1.0} & \multirow{2}{*}{0.5} & D-learning & \multirow{2}{*}{2.0} & \multirow{2}{*}{2.0} & 115 \\
& & & Reptile-D-learning & & & 58 \\
\midrule
\multirow{2}{*}{1.0} & \multirow{2}{*}{2.0} & \multirow{2}{*}{0.5} & D-learning & \multirow{2}{*}{-2.0} & \multirow{2}{*}{2.0} & 176 \\
& & & Reptile-D-learning & & & 106 \\
\midrule
\multirow{2}{*}{1.0} & \multirow{2}{*}{1.0} & \multirow{2}{*}{0.8} & D-learning & \multirow{2}{*}{2.0} & \multirow{2}{*}{-2.0} & 119 \\
& & & Reptile-D-learning & & & 49 \\
\bottomrule
\end{tabular}%
}
\end{table}

\subsection{Controller Performance and Generalization}
In this subsection, we evaluate Reptile-D-learning from two perspectives: (i) controller performance under nominal dynamics, and (ii) generalization under parameter shifts. All methods are tested on in-range settings with the same evaluation budget. During training, inverted-pendulum parameters are sampled from $m\in[0.5,2.0]$, $L\in[0.5,2.0]$, and $b\in[0.2,0.8]$; for the single-track car and UAV, uncertain parameters are varied within $\pm 50\%$ of nominal values.

\noindent\textbf{Controller Performance.}
Under nominal parameters, the Reptile-based controller achieves faster convergence and smaller steady-state error than the baselines while preserving Lyapunov-oriented stability constraints. For the inverted pendulum, the phase-portrait comparison in Fig.~\ref{fig:inv_compare_new} indicates faster closed-loop stabilization with Reptile-D-learning. For the single-track car, Fig.~\ref{fig:stcar_compare_new} shows both lower tracking error and improved trajectory quality under identical conditions. For the Crazyflie UAV, Fig.~\ref{fig:uav_compare_new} further demonstrates faster stabilization and shorter transient trajectories toward the target.
Quantitative convergence-step results on the in-range inverted-pendulum settings are reported in Table~\ref{tab:training_efficiency}.

We further visualize the RoA of the inverted pendulum in Fig.~\ref{fig:roa_inv_new} to interpret the observed performance gains. Reptile-D-learning yields a larger and steeper RoA, which provides a stronger Lyapunov-guided optimization signal and leads to faster closed-loop response and convergence.

\begin{figure*}[t]
    \centering
    \begin{minipage}{0.49\textwidth}
        \centering
        \includegraphics[width=\linewidth]{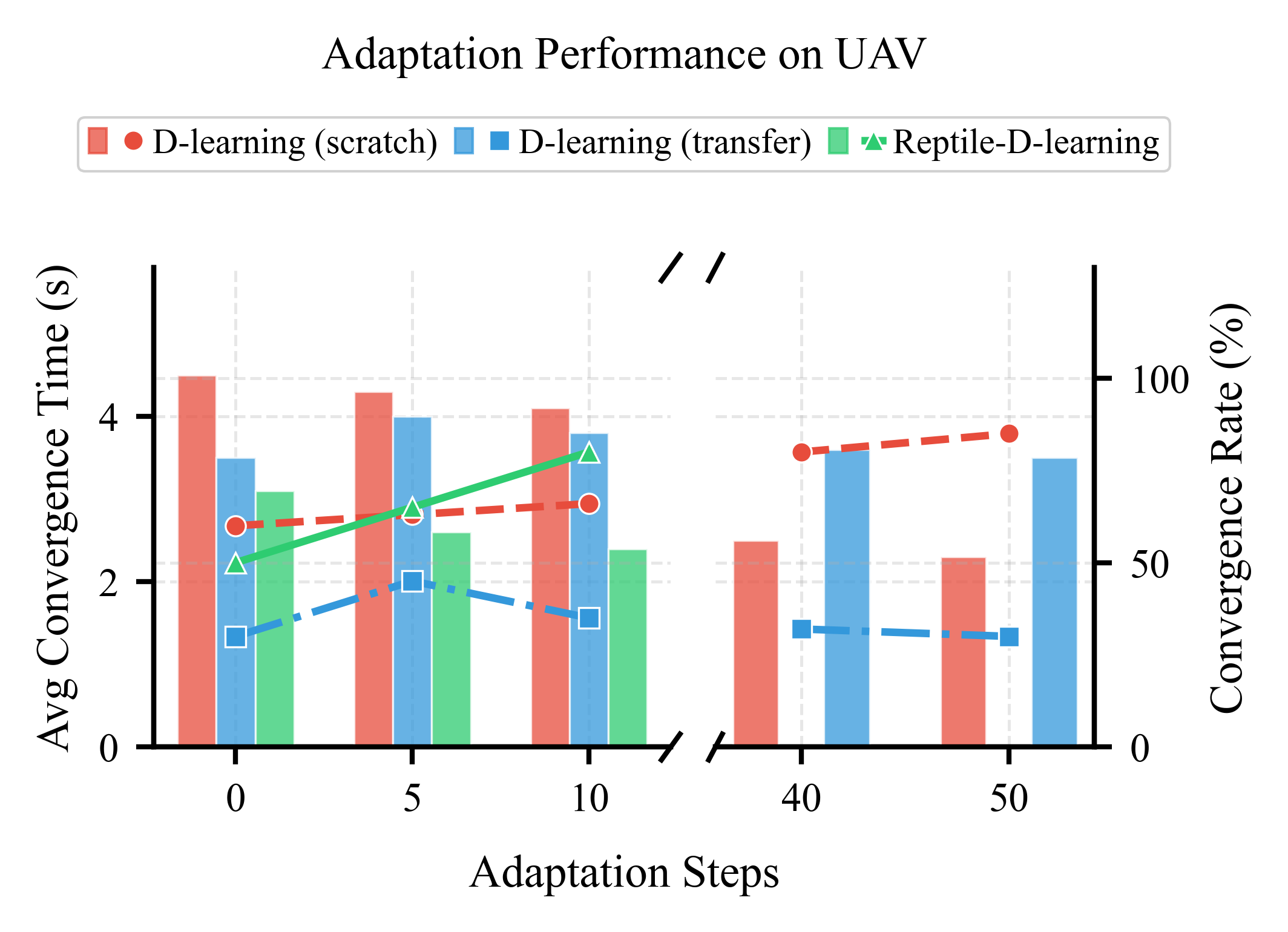}
        \captionof{figure}{Adaptation performance on UAV under OOD mass shift. Reptile-D-learning shows better convergence-time and success-rate trade-offs, especially in the few-step regime.}
        \label{fig:uav_adaptation_comparison}
    \end{minipage}
    \hfill
    \begin{minipage}{0.49\textwidth}
        \centering
        \includegraphics[width=\linewidth]{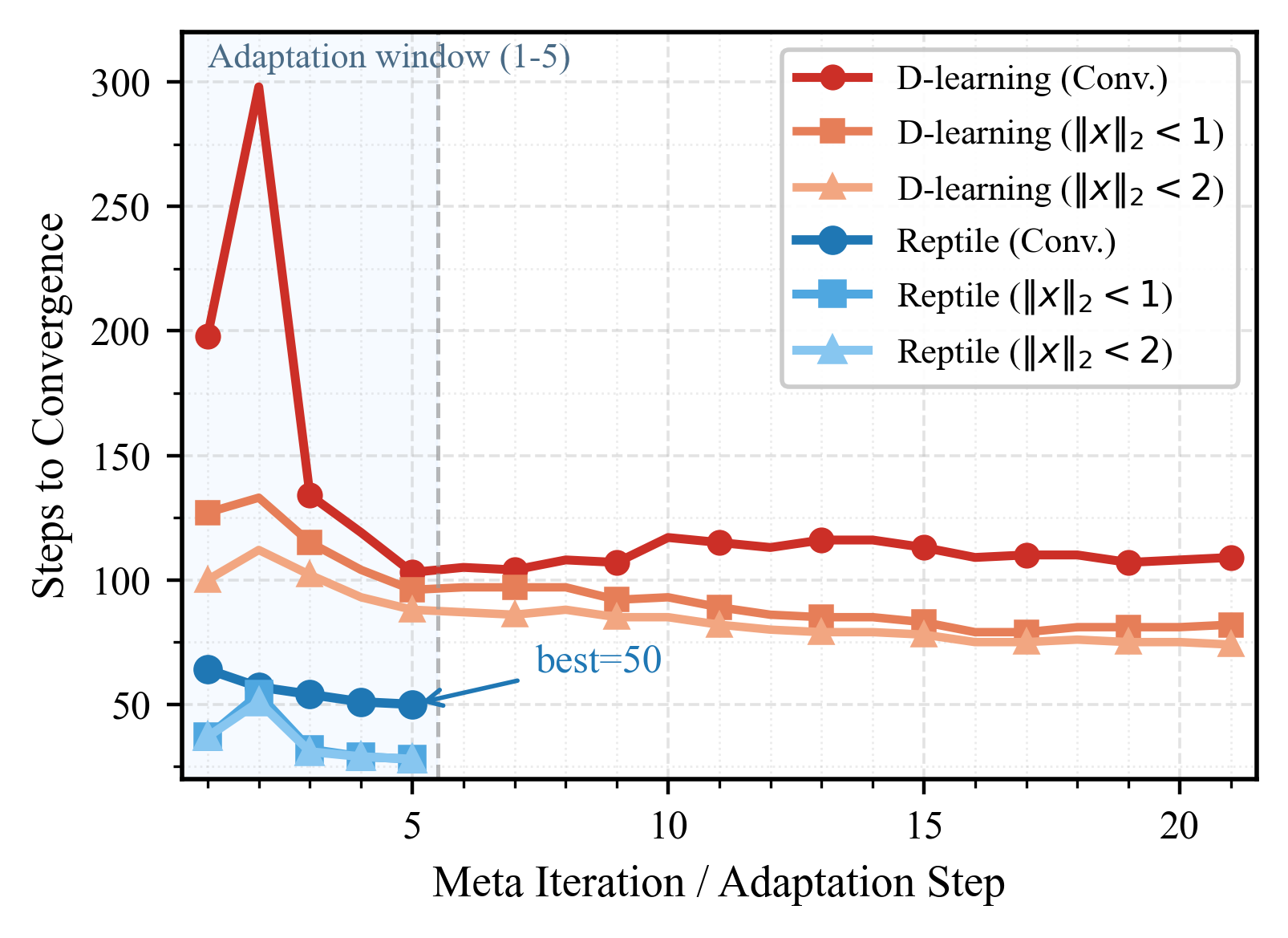}
        \captionof{figure}{Adaptation-window analysis. Within early adaptation steps, Reptile initialization reaches significantly fewer steps-to-convergence than standard D-learning across convergence criteria.}
        \label{fig:meta_vs_dlearning}
    \end{minipage}
\end{figure*}
\begin{figure*}[t]
    \centering
    \includegraphics[width=\textwidth]{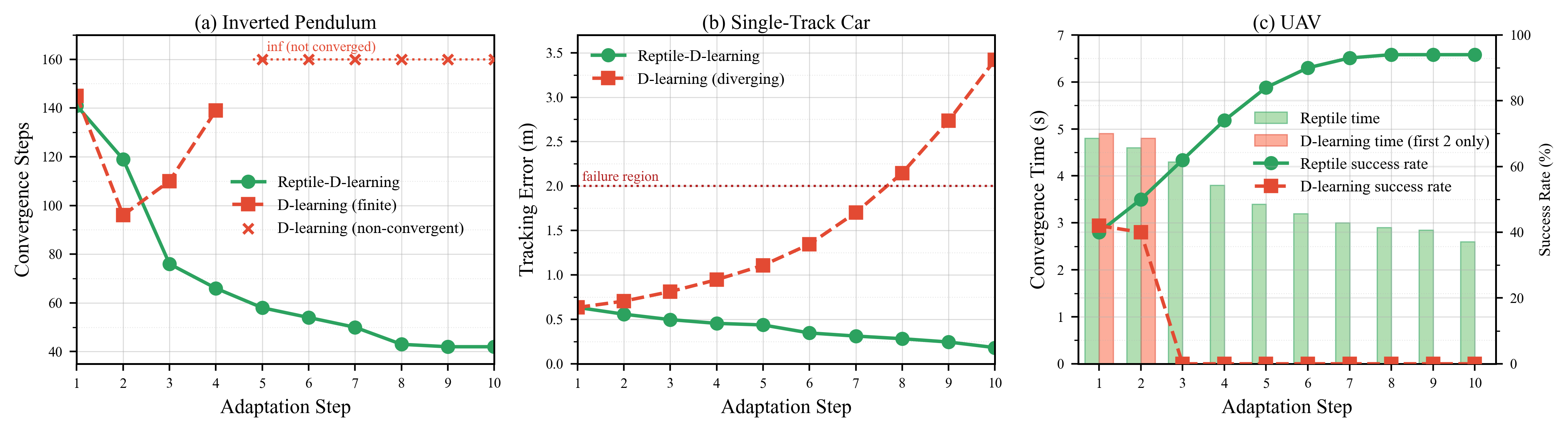}
    \caption{Unified ablation study across three systems. (a) Inverted pendulum: standard D-learning becomes non-convergent after early iterations, while Reptile-D-learning rapidly reduces convergence steps. (b) Single-track car: Reptile-D-learning steadily lowers tracking error, whereas D-learning diverges to meter-level error. (c) UAV: Reptile-D-learning shows decreasing convergence time and increasing success rate; D-learning maintains performance only in the first two rounds and then collapses to zero success rate.}
    \label{fig:adaptation_three_systems}
\end{figure*}

\noindent\textbf{Cross-Parameter Generalization.}
When parameters move away from the training distribution, baseline methods degrade significantly and may diverge, especially under severe shifts. In contrast, Reptile-D-learning maintains stronger stabilization and tracking performance, suggesting that the learned initialization captures transferable shared dynamics rather than overfitting to a single nominal model. This advantage is most evident in the UAV mass-shift case shown in Fig.~\ref{fig:hero}: standard D-learning often fails from challenging initial states, whereas Reptile-D-learning remains robust. Quantitative trends under shifted dynamics are further shown in Fig.~\ref{fig:tracking_generalization} and Fig.~\ref{fig:mass_robustness}.

\noindent\textbf{Overall Trend.}
Across all systems, the results show a consistent pattern: Reptile-D-learning improves both control quality and robustness to parameter mismatch. Quantitatively, it provides (a) higher stabilization success rates, (b) faster convergence, and (c) lower final tracking error under shifted dynamics. These observations support the effectiveness of jointly meta-learning the Lyapunov network, D-network, and policy for cross-parameter control.

\subsection{Adaptation Performance}
This subsection evaluates post-adaptation performance under limited adaptation steps. The core question is whether Reptile initialization can recover stability or achieve better control quality on OOD parameterized systems with fewer updates.

\noindent\textbf{Experimental Setup.}
We evaluate adaptation on two systems: the inverted pendulum and a UAV, both under fixed OOD settings. For the UAV, the mass and principal moments of inertia are set to 1.5$\times$ their nominal values. For the inverted pendulum, we use $m=3.0$, $L=1.0$, and $b=1.0$. Under these OOD configurations, we compare standard D-learning and Reptile-D-learning as a function of adaptation steps.

\noindent\textbf{Key Results.}
Under OOD dynamics, Reptile-D-learning consistently provides clear advantages in the few-step regime. For the inverted pendulum (Fig.~\ref{fig:meta_vs_dlearning}), Reptile reaches stable performance after only five adaptation iterations, whereas standard D-learning still cannot match it after 20 iterations. For the higher-dimensional UAV task (Fig.~\ref{fig:uav_adaptation_comparison}), Reptile-D-learning rapidly improves both task success rate and stabilization speed within ten adaptation steps, reaching a level that system-specific D-learning requires about 50 iterations to achieve. By contrast, transfer-and-finetune D-learning remains below this level even after more than 50 adaptation steps. These results indicate that the meta-initialization captures cross-task shared structure, so the inner loop mainly calibrates residual dynamics rather than relearning a full policy from scratch.
The quantitative adaptation trends are shown in Fig.~\ref{fig:uav_adaptation_comparison} and Fig.~\ref{fig:meta_vs_dlearning}.

\subsection{Ablation Study}

This subsection is designed to verify that the observed performance gains are attributable to Reptile-based meta-initialization, rather than to standard D-learning alone.

\noindent\textbf{Ablation results.}
Under the same training budget and the same degree of parameter randomization, we compare Reptile-D-learning (with Reptile meta-initialization) against standard D-learning. Results show that, once parameter randomization is introduced, standard D-learning fails to converge reliably across all three systems. Moreover, the divergence becomes more severe as system complexity and the number of varying parameters increase. These observations indicate that Reptile meta-initialization more effectively captures and preserves task-shared dynamical structure, thereby maintaining controller generality and stabilizing the training process.
A unified cross-system adaptation comparison is provided in Fig.~\ref{fig:adaptation_three_systems}.

\noindent\textbf{Computational cost (Reptile vs. MAML).}
Compared with other meta-learning methods, Reptile provides a clear computational advantage for D-learning's coupled loss structure. Under the same training configuration, Reptile typically requires no more than 5\,GB of GPU memory, while MAML requires second-order gradients and full computation-graph retention, resulting in memory demand above 900\,GB. Consequently, comparable MAML experiments are infeasible on an RTX 4080 SUPER (16\,GB). Under this practical constraint, Reptile is the most feasible meta-learning framework for D-learning.

\section{CONCLUSION}

We presented Reptile-D-learning, a unified bilevel framework that integrates D-learning and first-order meta-optimization for parametric adaptation in Lyapunov-based control.
By treating the Lyapunov network, Dfunction network, and controller as a single meta-parameter and optimizing post-adaptation performance across tasks, the method reduces the appearance of modular stacking and provides a coherent optimization view.
Future work will focus on tighter non-asymptotic analysis of adaptation error under distribution shift and hardware validation on real robotic platforms.

\bibliographystyle{IEEEtran}
\bibliography{refs}

\end{document}